\documentclass[sigconf]{acmart}
\AtBeginDocument{%
  \providecommand\BibTeX{{%
    \normalfont B\kern-0.5em{\scshape i\kern-0.25em b}\kern-0.8em\TeX}}}

\setcopyright{acmcopyright}
\copyrightyear{2023}
\acmYear{2023}
\acmDOI{XXXXXXX.XXXXXXX}

\acmConference[ICAIF '23]{4th ACM International Conference on AI in Finance
}{November 27--29,
  2023}{New York City}
\acmPrice{15.00}
\acmISBN{978-1-4503-XXXX-X/18/06}

\usepackage{hyperref}
\usepackage{mathtools}
\usepackage{tabularx,ragged2e}
\usepackage{verbatim}
\usepackage{diagbox}
\usepackage{thmtools, thm-restate}
\usepackage{multicol, slashbox}
\usepackage{booktabs}
\usepackage{cuted}

\usepackage{natbib}
\DeclareRobustCommand{\E}[0]{\mathbb{E}}
\DeclareRobustCommand{\F}[0]{\mathcal{F}}

\DeclareRobustCommand{\N}[0]{\mathbb{N}}

\DeclareRobustCommand{\P}[0]{\mathbb{P}}

\DeclareRobustCommand{\R}[0]{\mathbb{R}}

\DeclareRobustCommand{\x}[0]{\mathrm{x}}
\DeclareRobustCommand{\u}[0]{\mathrm{u}}
\DeclareRobustCommand{\y}[0]{\mathrm{y}}
\newtheorem{theorem}{Theorem}
\newtheorem{lemma}{Lemma}
\newtheorem{proposition}{Proposition}
\newtheorem{definition}{Definition}
\newtheorem{corollary}{Corollary}
\newtheorem{remark}{Remark}

\DeclareMathOperator{\Sig}{Sig}

\DeclareMathOperator{\softmax}{\Phi}

\NewDocumentCommand{\evalat}{sO{\big}mm}{%
  \IfBooleanTF{#1}
   {\mleft. #3 \mright|_{#4}}
   {#3#2|_{#4}}%
}

\begin{document}

\title{Sig-Splines: universal approximation and convex calibration of time series generative models}

\author{Magnus Wiese}
\email{magnus@premia.finance}
\affiliation{%
  \institution{University of Kaiserlautern}
  \streetaddress{Gottlieb-Daimler Straße 48}
  \city{Kaiserslautern}
  \country{Germany}
}

\author{Phillip Murray}
\email{phillip.murray18@imperial.ac.uk}
\affiliation{%
  \institution{Imperial College London}
  \streetaddress{180 Queen’s Gate}
  \city{London}
  \country{United Kingdom}
}

\author{Ralf Korn}
\email{korn@mathematik.uni-kl.de}
\affiliation{%
\institution{University of Kaiserlautern}
\streetaddress{Gottlieb-Daimler Straße 48}
\city{Kaiserslautern}
\country{Germany}
}

\renewcommand{\shortauthors}{Wiese, Murray, and Korn}

\begin{abstract}
  We propose a novel generative model for multivariate discrete-time time series data. 
  Drawing inspiration from the construction of neural spline flows, our algorithm incorporates linear transformations and the signature transform as a seamless substitution for traditional neural networks. This approach enables us to achieve not only the universality property inherent in neural networks but also introduces convexity in the model's parameters. 
\end{abstract}



%
\maketitle

\section{Introduction}
\label{sec:introduction}
Constructing and approximating generative models that fit multidimensional time series data with high realism is a fundamental problem with applications in probablistic prediction and forecasting \cite{rasul2020multivariate} and in synthetic data simulation in areas such as finance \cite{arribas2020sigsdes, horvath2020, buehler2022deep, hao2020, SigWGAN, wiese2021multiasset, wiese2020quant}. The sequential nature of the data presents some unique challenges. It is essential that any generative model captures the temporal dynamics of the time series process - that is, at each time step, it is not enough to simply reflect the marginal distribution $p(\mathbf{x}_t)$ but instead we need to model the conditional density $p(\mathbf{x}_{t+1} | \F_t)$ where the filtration $\F_t$ represents all information available at time $t$. Therefore, the model requires an efficient encoding of the history of the path as the conditioning variable.

In this paper, we present a novel approach to time series generative modelling that combines the use of \emph{path signatures} with the framework of normalizing flow density estimation. Normalizing flows \cite{kobyzev2020normalizing, papamakarios2019normalizing} are a class of probability density estimation models that express the data $\mathbf{x}$ as the output of an invertible, differentiable transformation $T$ of some base noise $\mathbf{u}$:
\begin{equation*}
    \mathbf{x} = T(\mathbf{u}) \qquad \text{where} \ \mathbf{u} \sim p(\mathbf{u})
\end{equation*}

A special instance of such models is the \emph{neural spline flow} \cite{durkan2019neural} which construct a triangular chain of approximate conditional CDFs of the components of $\mathbf{x}$, by using neural networks to output the width and height of a sequence of \emph{knots} that are used to construct a monotonically increasing CDF function for each dimension, through spline interpolation which could be linear, rational quadratic or cubic \cite{durkan2019neural, durkan2019cubic}. Such models may be estimated by finding neural network parameters that minimize the forward Kullback-Leibler divergence between the target density $p(\mathbf{x})$ and the flow-based model density $p_\theta(\mathbf{x})$. Thus, training such models may face the known challenges in neural netowrk training. In particular, learning the conditional density for time series data $p(\mathbf{x}_{t+1}| \F_t)$ would in general require some recurrent structure in the neural network to account for the path sequence, unless some specific Markovian assumptions were made.

To overcome these challenges, we replace the neural network in the neural spline flow with an alternative function approximator, via the signature transform. First introduced in the context of rough path theory \cite{friz2020course, lyons2014rough}, the signature gives an efficient and parsimonious encoding of the information in the path history (i.e. the filtration). This encoding provides a feature map which exhibits a \emph{universal approximation property} - any continuous function of the path can be approximated arbitrarily well by a linear combination of signature features. This makes signature-based methods highly computationally efficient for convex objective functions.       


By incorporating the signature transform, our algorithm, termed the \emph{signature spline flow}, possesses two significant properties:
\begin{itemize}
    \item \textit{Universality}: We leverage the universal approximation theorem of path signatures, allowing our model to approximate the conditional density of any time series model with arbitrary precision.
    \item \textit{Convexity}: By replacing the neural network with the signature transform, we demonstrate that our optimization problem becomes convex in its parameters. This means that gradient descent or convex optimization methods lead to a unique global minimum (except for potential linearly dependent terms within the signature).
\end{itemize}
\section{Related work}
\label{sec:related_work}
The seminal paper of \citeauthor{GAN} on Generative Adversarial Networks (\emph{GANs}) \cite{GAN} has led to a rapid expansion of the literature related to generative modelling. While the algorithm may seem blatant at first different optimisation techniques have led to the convergence towards equilibria that give satisfying results for types modalities of data including images \cite{BIGGAN}, (financial) time series data \cite{wiese2020quant, TIMEGAN}, music \cite{GANSYNTH} and text to image synthesis \cite{reed2016generative}.

Normalizing flows are an alternative approach to constructing a generative model by creating expressive bijections allowing for tractible conditional densities. In a somewhat chronological naming the most impactful papers, NICE \cite{NICE} proposed additive coupling layers and real-valued non-volume preserving (real NVP) transformations authored by \citeauthor{REALNVP} presented chaining affine coupling transforms for constructing expressive bijections which were shown to be universal \cite{UNIVERSALCOUPLINGTRANSFORMS}. A bit later the idea of constructing a triangular map and leveraging the theorem of \citeauthor{bogachev2005triangular} \cite{bogachev2005triangular} got popular and coined by various authors \emph{autoregressive flow}. Works such that presented algorithms that leverage autoregressive flows include \citeauthor{MAF} \cite{MAF}, \citeauthor{wehenkel2019unconstrained} \cite{wehenkel2019unconstrained} and finally \citeauthor{durkan2019neural} \cite{durkan2019cubic,durkan2019neural} in \citeyear{durkan2019neural}. While the invertibility property of normalizing flows at first sight can seem limiting as a flow can not construct densities on a manifold, this problem was addressed with Riemannian manifold flows \cite{gemici2016normalizing}, where an injective decoder mapping is consstructed to the high-dimensional space. An efficient algorithm to guarantee the injectiveness was introduced in \cite{brehmer2020flows}.

Our research paper centers around constructing a conditional parametrized generative density $p_\theta(\mathbf{x}_{t+1} | \F_t)$ for time series data, where the condition is specified by the current filtration $\F_t$. Previous studies have explored the estimation and construction of such densities using real data.

Ni et al. \cite{hao2020} propose a conditional signature-based Wasserstein metric and estimate it through signatures and linear regression. The combination of GANs and autoencoders is utilized by \cite{TIMEGAN} to learn the conditional dynamics. Wiese et al. \cite{wiese2022risk, wiese2021multiasset} employ neural spline flows as a one-point estimator for estimating the conditional law of observed spot and volatility processes.

Other approaches focus on estimating the unconditional law of stochastic processes using Sig-SDEs \cite{arribas2020sigsdes}, neural SDEs \cite{gierjatowicz2020robust}, Sig-Wasserstein metric \cite{SigWGAN}, and temporal convolutional networks \cite{wiese2020quant}. Among these, Dyer et al. \cite{dyer2021deep} have work most closely related to ours, utilizing deep signature transforms \cite{NEURIPS2019_d2cdf047} to minimize the KL-divergence.

\section{Signatures}
\label{sec:signatures}
When working with time series data in real-world applications we typically have a dataset which may have been sampled at discrete and potentially irregular intervals. Whilst many models take the view that the underlying process is a discrete-time process, the perspective in rough path theory is to model the data as discrete observations from an unknown continuous-time process. The trajetories of these continuous-time processes define paths in some path space. Rough path theory, and \emph{signatures} in particular, give a powerful and computationally efficient way of working with data in the path space. 

First introduced by Chen \cite{chen1957integration, chen2001iterated}, the signature has been widely used in finance and machine learning. We provide here only a very brief introduction to the mathematical framework we will be using to define the signature, and refer the reader to e.g. \cite{chevyrev2016primer, lyons2007differential} for a more thorough overview.

The signature of a $V$-valued path takes values in the tensor algebra space of $V$. In this section, we assume for simplicity that $V$ is the real-valued Euclidean $d$-dimensional vectors space; i.e. $V=\R^d$.

We now may define the signature of a $d$-dimensional path, as a sequence of iterated integrals which exists in the tensor algebra $T((\mathbb{R}^d))$. 

\begin{definition}[Signature of a path]\label{sig-definition}
    Let $\mathbf{X} =(\mathbf{X}_1, \ldots, \mathbf{X}_d) \colon [0, 1] \to \mathbb R^d$ be continuous.
    The \emph{signature of $X$ evaluated at the word $\mathbf{i} = (i_1, \dots, i_k) \in [d]^{\times k}$} is defined as 
    \begin{equation*}
        \Sig_{\mathbf{i}}(\mathbf{X}) = \underset{0 < t_1 < \cdots < t_k < 1}{\int\cdots\int} \mathrm dX_{t_1, i_1} \cdots \mathrm dX_{t_k, i_k} \in (\R^d)^{\otimes k}
    \end{equation*}
    Furthermore, the \emph{signature of $\mathbf x$} is defined as the collection of iterated integrals
    \begin{equation*}
        \Sig(\mathbf x) = \left(\Sig_{\mathbf{i}}(\mathbf x)\right)_{\mathbf{i} \in [d]^\star} \in T((\R^d)).
    \end{equation*}
\end{definition}
\begin{remark}[Signature of a sequence]
    Let $\mathbf x = (\mathbf{x}_1, \ldots, \mathbf{x}_n) \in S(\R^d)$ be a sequence. Let $\mathbf X =(\mathbf X_1, \ldots, \mathbf X_d) \colon [0, 1] \to \mathbb R^d$ be continuous, such that $\mathbf X(\tfrac{i -1}{n - 1}) = \mathbf{x}_i$, and linear on the intervals in between. For convenience, we denote the signature of the linearly embedded sequence $\mathbf{x}$ for any word $ \mathbf{i} = (i_1, \dots, i_k) \in [d]^\star $ by $\Sig_{\mathbf{i}}(\mathbf x) = \Sig_{\mathbf{i}}(\mathbf X)$. 
\end{remark}

The signature is a projection from the path space into a sequence of statistics of the path. Therefore, it can informally be thought of as playing the role of a basis on the path space. Each element in the signature encodes some information about the path, and moreover, some have clear interpretations, particularly in the context of financial time series - the first term represents the increment of the path over the time interval, oftern referred to as the \emph{drift} of the process. Different path transformations, such as the \emph{lead-lag transformation} \cite[Page 20.]{chevyrev2016primer} can be applied before computing the signature to obtain statistics such as the \emph{realized volatilty}.  

To state the universality property of signatures we have to introduce the concept of time augmentation. Let $BV([0, T], \R^{d+1})$ denote the space of $\R^d$-valued paths with bounded variation. Let $\tilde{\psi}: BV([0, T], \R^{d}) \to BV([0, T], \R^{d+1})$ be a function which is defined for $t \in [0, T]$ as $\tilde\psi(\mathbf{X})(t) = (t, X(t)) $. 
We call 
\begin{align*}
    &\Omega_0([0, 1]; \R^d) \coloneqq \\&\lbrace \mathbf{X}: [0, T] \to \R^d \ | \ \tilde{\mathbf{X}} \in BV([0, T], \R^d), \tilde{\mathbf{X}}(0) = 0, \mathbf{X}=\tilde\psi(\tilde{\mathbf{X}}) \rbrace
\end{align*}
the \emph{space of time-augmented paths starting at zero}. We have the following representation property of signatures.

\begin{proposition}[Universality of signatures]\label{prop:universal}
	Let $F$ be a real-valued continuous function on continuous piecewise smooth paths in $\mathbb R^d$ and let $\mathcal K \subset \Omega_0([0, 1]; \R^d)$ be a compact set of such paths. 
    Let $\varepsilon >0$. Then there exists an order of truncation $L \in \N$ and a linear functional $\mathbf{w} \in T((\R^{d+1})^*) $ such that for all $\hat{\mathbf{X}} \in \mathcal K$,
	\begin{equation*}
        \left\vert F(\mathbf{X}) - \langle \mathbf{w}, \Sig^{(L)}(\hat{\mathbf{X}}) \rangle \right\vert < \varepsilon
	\end{equation*}
\end{proposition}

This result shows that any continuous function on a compact set of paths can be approximated arbitrarily well simply by a linear combination of terms of the signature. For a proof, see \cite{kiraly2019kernels}. It can be thought of as a universal approximation property akin to the oft-cited one for neural networks \cite{cybenko1989approximation}. The primary difference is that neural networks typically require a very large number of parameters which must be optimized through gradient-based methods, of a loss function which is  non-convex in the network parameters, whereas the approximation via signatures amounts to a simple linear regression in the signature features. Hence, once the signature is calculated, function approximation becomes extremely computationally efficient and can be done via second order methods. Furthermore, as we will see later, we may apply further convex transformations of the signature and retain an objective that is convex in the parameters.

\section{Linear neural spline flows}
\label{sec:neural_spline_flows}
In this section, we revisit neural spline flows as prerequisites. Specifically, we explore their application to multivariate data in \autoref{sec:multivariate_spline} and further delve into their relevance in the context of multivariate time series data in \autoref{sec:neural_spline_flow_discrete_time}.

\subsection{Multivariate data}
\label{sec:multivariate_spline}
Without loss of generality let $\mathbf{x}=(\x_1, \dots, \x_d) \sim p$ be a $[0, 1]^d$-valued random variable.\footnote{Note that any random variable can be defined on the $[0, 1]$ by applying the integral probability transform (IPT). Note that this any bijection applied to make a random variable $[0, 1]^d$-valued does not impact the KL-divergence objective that is later derived.
} Furthermore, denote by $F_1(x)\coloneqq \P(\x_1 \leq x)$ the cumulative distribution function (\emph{CDF}) of $\x_1$ and for $i \in \lbrace 2, \dots, d \rbrace$ by 
$$
F_i({x} | \mathbf{x}_{:i-1}) \coloneqq \P( \x_i\leq x | {\x}_{:i-1}=\mathbf{x}_{:i-1})
$$ 
the \emph{conditional CDF} of $\x_i$ given that the value of ${\x}_{:i-1}$ is $\mathbf{x}_{:i-1}$. The set of conditional CDFs $\mathbf{F} = (F_1, \dots, F_d)$ completely defines the joint distribution of the random variable $\mathbf{x} \sim p$. For completeness we restate the inverse sampling theorem in multiple dimensions:
\begin{theorem}[Inverse sampling theorem]
    Let $\mathbf{u} = (\u_1, \dots, \u_d)$ be a uniformly distributed random variable on $[0, 1]^d$. Furthermore, let $\y_1 = (F^1)^{-1}(\u_1)$ and define for $i = 2, \dots, d$ the random variables 
    $$ \y_i = (F_i)^{-1}(\u_i | \y_1, \dots, \y_{i-1}).$$
    Then the random variables $\mathbf{x} \sim p$ and $\mathbf{y}=(\y_1, \dots, \y_d)$ are equal in distribution, i.e. $\mathbf{x} \stackrel{d}{=} \mathbf{y}$.
\end{theorem}
\begin{proof}
    See \cite[Section 2.2]{papamakarios2019normalizing} for a derivation or \cite{bogachev2005triangular} for a more formal treatment of the proof. 
\end{proof}

Various methods exist to approximate a set of conditional CDFs $\mathbf{F} = (F_1, \dots, F_d)$. In this paper, we are interested in a spline-based approximation 
$$\hat{F}_i: [0,1] \times [0, 1]^{i-1} \times \Theta \to [0,1], \ i \in [d]$$
parametrised by model parameters $\theta \in \Theta$ and the conditioning vector $\mathbf{x}_{:i-1}$. Crucially, the constructed spline has to satisfy the properties of a CDF: it has to be (1) monotonically increasing $x_i$ and (2) span from $0$ to $1$. In order to satisfy both requirements the following construction is used. 

Let $N + 1 \in \N$ be the number of knots used to construct the spline. Furthermore, denote by $\Phi: \R^N \to \R^N$ the \emph{softmax transform}, which is defined as $\Phi(\mathbf{x}) = \left({\exp(\mathbf{x}_j)}/{\sum_{i=1}^N \exp(\mathbf{x}_i)}\right)_{j\in\lbrace 1, \dots, N \rbrace} $, and let $H_i: [0, 1]^{i-1} \times \Theta \to \R^N$ be a potentially non-linear function which we shall coin the \emph{feature map}. We call the composition of the softmax transform with the feature transform $H_i$ the \emph{increment function} and denote it by $$ \Delta_i(\mathbf{x}_{:i-1}, \theta) = \Phi \circ H_i(\mathbf{x}_{:i-1}, \theta) \ . $$ The parametrised spline with linear interpolation $\hat{F}_i$ is then defined as 
\begin{align*}
    &F_i(x_i; \mathbf{x}_{:i-1}, \theta) 
    =\\ &
    \begin{cases}
        \sum_{j=1}^{k}\Delta_j(\mathbf{x}_{:i-1}, \theta) 
        + (x_i - \frac{k}{N})
        \dfrac{\Delta_{k+1}(\mathbf{x}_{:i-1}, \theta)}{\frac{1}{N}}
            & \ \textrm{if} 
        \ x \in [\frac{k}{N}, \frac{k+1}{N}]
    \end{cases}
\end{align*}
Note that due to the application of the softmax transform the constructed spline satisfies the properties of a CDF. 

Using the above definition of a linear spline CDF we can construct a \emph{linear spline flow} which approximates the set of CDFs $\mathbf{F}$: 
\begin{definition}[Spline flow with linear interpolation]
    For $i \in [d]$ let ${F}_{i}: [0, 1] \times [0, 1]^{i-1} \times \Theta^{} \to [0, 1]$ be linear spline CDFs. Then we call the function $\mathbf{F}: [0, 1]^d \times \Theta^d \to [0, 1]^d$ defined as 
    \begin{equation*}
        {\mathbf{F}}(\mathbf{x}; \theta) = ({F}_{i}(x_i ; \mathbf{x}_{:i-1}, \theta_{i}))_{i \in [d]}
    \end{equation*}
    where $\theta = (\theta_{i})_{i \in [d]}$, \emph{neural spline flow with linear interpolation}. 
\end{definition}

Various options for the feature transform $H_i$ exist. The most wide-spread in Machine Learning literature is a neural network, resulting in a \emph{linear neural spline flow}. 
\begin{definition}[Linear neural spline flow]
    For $i \in [d] $ let $H_i = \mathcal{NN}: [0, 1]^{i-1} \times \Theta \to [0, 1]^N$ be a neural network. We call a spline flow $\mathbf{F}: [0, 1]^d \times \Theta^d \to [0, 1]^d$ taking as a feature map $(H_i)_{i \in [d]}$ a \emph{neural spline flow}. 
\end{definition}
\begin{remark}[Interpolation schemes]
    Other interpolation schemes that ensure that ${F}_i$ is monotonically increasing exist (see for example \cite{durkan2019neural}). For the sake of this paper, we will restrict ourselves to linear interpolation. 
\end{remark}

Utilizing normalizing flows, such as linear spline CDFs, offers the advantage of allowing for analytical evaluation of the likelihood function. This characteristic can be expressed explicitly as an \emph{autoregressive flow} \cite{papamakarios2019normalizing}, employing the chain rule of probability (defining $p(x_1 | \mathbf{x}_{:0}) = p(x_1)$):
\begin{equation}
    \label{eq:ns_density_1}
    p_\theta(\mathbf{x}) = \prod_{i=1}^d p_\theta(x_i | \mathbf{x}_{:i-1}) 
\end{equation}
The density can be further expanded by using the definition of a parametrised linear CDF 
\begin{align}
    \label{eq:ns_density_2}
    p_\theta(\mathbf{x}) &= 
    \prod_{i=1}^d \prod_{k=1}^{N} \left(N{\Delta_{i, k}(\mathbf{x}_{:i-1}, \theta)}\right)^{C_{i,k}(\mathbf{x})} \\  
    &= N^d \prod_{i=1}^d \prod_{k=1}^{N} {\Delta_{i, k}(\mathbf{x}_{:i-1}, \theta)}^{C_{i,k}(\mathbf{x})}
\end{align}
where $C_{i,k}(\mathbf{x}) = \mathbf{1}_{\lbrace x_i \in [\frac{k-1}{N}, \frac{k}{N}) \rbrace}$ is the indicator function for the $k$-the bin which is $1$ if $x_i$ falls in to the bin $[\frac{k-1}{N}, \frac{k}{N})$ and $0$ else. 
Due to the tractability of the conditional density the calibration of the parameters $\theta \in \Theta$ can be performed by minimizing the Kullback-Leibler (KL) divergence. The KL-divergence is defined as the expected ratio of the log-densities under the \emph{true} density $p$
\begin{equation*}
    \mathrm{KL}(p, p_\theta) \coloneqq \E_p\left[\ln \dfrac{p(\mathbf{x})}{p_\theta(\mathbf{x})}\right]
\end{equation*}
and for a neural spline flow $\mathbf{F}_\theta$ with linear interpolation enjoys the explicit representation: 
\begin{align*}
    \mathrm{KL}(p, p_\theta) &= -\E_p\left[\ln p_\theta(\mathbf{x})\right] + K \\
    &= -\sum_{i=1}^d\sum_{k=1}^{N}\E_p\left[
        C_{i, k}(\mathbf{x}) \ln \Delta_{i, k}(\mathbf{x}_{:i-1}, \theta)
    \right] + K 
\end{align*}
where $K$ represents a constant term that is independent of the parameters $\theta \in \Theta$. 

In practice, the density $p$ is generally unknown and the expectation in the above derived KL-divergence needs to be estimated via Monte Carlo (MC) for a sample ${\lbrace \mathbf{x}^{(j)} \rbrace_{j=1}^M \sim p}$ of size $M$. In this case, the MC approximation is given as
\begin{align*}
    J(\theta) \coloneqq
    - \dfrac{1}{M} \sum_{i=1}^{d}\sum_{j=1}^M \ln p_\theta\big(x_i^{(j)} | \mathbf{x}_{:i-1}^{(j)}\big)
\end{align*}
where we drop the constant $const$ term. The calibration problem is then defined as the minimization of the loss function $J:\Theta \to \R$ 
\begin{equation*}
    \min_{\theta \in \Theta} J(\theta). 
\end{equation*}
Updates of the parameters are performed via \emph{batch} or \emph{stochastic gradient descent} \cite{ruder2016overview} for $k \in \lbrace 1, \dots, N \rbrace $ and initial parameters $\theta^{(0)} \in \Theta^{\times 2} $
\begin{equation*}
    \theta^{(k+1)} \gets \theta^{(k)} - \alpha \nabla_\theta \evalat{J}{\theta=\theta^{(k)}}
\end{equation*}
where $\alpha > 0$ is the \emph{step size} or \emph{learning rate} and the gradients of the neural network's parameters are computed via the \emph{backpropagation algorithm} \cite{rumelhart1986learning}.

\subsection{Discrete-time stochastic processes}
\label{sec:neural_spline_flow_discrete_time}
In the case of time series data one is interested in approximating the transition dynamics - that is, the conditional density of the next observation, conditional on past states of the time series. Let $(\Omega, \mathbb{F} = (\F_t)_{t \in \mathbb{N}}, \mathbb{P})$ be a filtered probability space and let $(\mathbf{x}_t)_{t} \sim p$ be a time series observed at discrete timestamps, which for ease of notation we assume to be regular but they need not be. As in the previous section, we assume without loss of generality that the time series $(\mathbf{x}_t)_{t} \sim p$ is $[0, 1]^d$-valued. Assume further that the process $\mathbf{x}\sim p$ is generative in the sense that the filtration is generated by the process $ \F_t = \sigma((\mathbf{x}_s)_{s=0, \dots, t}), t \in \N$. 

Our objective is to approximate a conditional model density $p_\theta(\mathbf{x}_{t+1} | \F_t)$ that minimizes an \emph{adapted version} of the KL-divergence to the true conditional density 
\begin{equation*}
    \mathrm{aKL}(p, p_\theta) = \E_p\left[\E_p\left[\ln \dfrac{p(\mathbf{x}_{t+1} | \F_{t})}{ p_\theta(\mathbf{x}_{t+1} | \F_{t})}\Big| \F_{t}\right] \right]
\end{equation*}
To accomodate for the conditional information of past samples, i.e. the filtration, we need to generalise neural spline flows by including the condition, i.e. past states of the time series.

\begin{definition}[Discrete-time linear spline flow]
    Fix $t \in \N$ and for $i \in [d]$ let $H_i: [0, 1]^{dt + (i-1)} \times \Theta \to \R^N$ be a non-linear function and ${F}_{i}: [0, 1] \times [0, 1]^{dt + (i-1)} \times \Theta \to [0, 1]$ be linear spline CDFs taking $H_i$ as the feature map. We call the function $\mathbf{F}: [0, 1]^d \times [0, 1]^{dt} \times \Theta^d \to [0, 1]^d$ defined as 
    \begin{equation}
        \label{eq:discrete_time_linear_spline_flow}
        {\mathbf{F}}(\mathbf{x}_{t+1}; \mathbf{x}_t, \theta) = ({F}_{i}(x_{t+1, i} ; (\mathbf{x}_{\leq t}, \mathbf{x}_{t+1, :i-1}), \theta_{i}))_{i \in [d]}
    \end{equation}
    a \emph{discrete-time linear spline flow}. 
\end{definition}
Following the defintion of a neural spline flow the discrete-time linear neural spline flow, simply is defined as a discrete-time linear spline flow where the feature maps $(H_i)_{i \in [d]}$ are neural networks.

\begin{remark}[Markovian dynamics]
\label{rem:markov}
In case the time series is Markovian with $r$-lagged memory a single conditional neural spline flow can be constructed to approximate the conditional law at any time $t \in \N$. 
The objective reduces to the ``simpler'' adapted KL-divergence 
\begin{align*}
    \mathrm{aKL}&(p, p_\theta) = \E_p\left[\E_p\left[\ln \dfrac{p(\mathbf{x}_{t+1} | \F_{t})}{ p_\theta(\mathbf{x}_{t+1} | \F_{t})}\Big| \F_{t}\right] \right] \\
    &= -\E_p\left[\E_p\left[\ln p_\theta(\mathbf{x}_{t+1} | \mathbf{x}_{t-r+1:t})\Big| \mathbf{x}_{t-r+1:t}\right] \right] + K
\end{align*}
\end{remark}

\section{Signature spline flows}
\label{sec:signature_spline_flows}
In this section, we introduce signature splines flows and adopt the notation from \autoref{sec:neural_spline_flow_discrete_time}. 
Before we proceed, we first define a couple of helper augmentations which will be useful to define the sig-spline. 
\begin{definition}[Mask augmentation]
    The function $m: S(\R^d) \times [d] \to S(\R^d)$ defined 
    for $i \in [d]$ as 
    \begin{equation*}
        m(\mathbf{x}, i) = \left(\mathbf{x}_1, \dots, \mathbf{x}_{n-1}, \tilde{m}(\mathbf{x}_{n}, i)\right)
    \end{equation*}
    where $\tilde{m}(x_n, i): \R^d \times [d] \to \R^d$ is defined as 
    $$\tilde{m}(x_n, i)= (x^1_n, \dots, x^{i-1}_n, x^{i}_{n-1}, \dots, x^{d}_{n-1})^T$$
    is called \emph{mask augmentation}. 
\end{definition}
Thus, the mask augmentation restrains any information beyond the $(i-1)^{th}$ coordinate of the $n^{th}$ term of a sequence. This is useful to define any conditional density approximator for discrete-time data. 

To obtain the universality property of signatures the sequence has to start at $\mathbf{0}$ and needs to be time-augmented. Both following definitions come useful: 
\begin{definition}[Basepoint augmentation]
    Let $\phi: S(\R^d) \to S(\R^d)$  defined as $\phi: (\mathbf{x}_1, \dots, \mathbf{x}_n) \mapsto (\mathbf{0}, \mathbf{x}_1, \dots, \mathbf{x}_n)$ \emph{basepoint augmentation}. 
\end{definition}
\begin{definition}[Time augmentation]
    \sloppy
    We call the function $\psi: S(\R^d) \to S(\R^{1+d})$ defined as 
    \begin{equation*}
    \psi: (\mathbf{x}_1, \dots, \mathbf{x}_n) \mapsto ((t_1, \mathbf{x}_1), \dots, (t_n, \mathbf{x}_n))
    \end{equation*}
    \emph{time augmentation}. 
\end{definition}
Last, let $\gamma_i: S(\R^d) \to S(\R^{d+1})$ be the composition of the basepoint, time and the mask augmentation defined as $ \gamma_i = \phi \circ m_i \circ \psi $ where $m_i$ is the mask augmentation applied at the $i^{th}$ coordinate. 

The signature spline flow is defined in the same spirit as the neural spline flow; except that the neural network-based CDF approximator is replaced by a signature-based one and augmentations $\gamma_i, i \in [d]$ are applied to the raw time series.  
\begin{definition}[Signature spline flow]
    Let $\mathbf{x} \in S(\R^d)$ be a discrete-time series of length $t \in \N$, $L$ be the order of truncation and let $\Theta = \times_{j=1}^N T^{(L)}((\R^{d+1})^*)$ be the parameter space. Furthermore, let $H_i: S([0, 1]^d) \times \Theta \to \R^N$ be defined as 
    \begin{equation*}
        H_i(\mathbf{x}, \mathbf{U})= \langle \mathbf{u}_j, \Sig^{(L)}(\gamma_i(\mathbf{x})) \rangle_{j \in [N]}
    \end{equation*}
    where $\mathbf{U} = (\mathbf{u}_{1}, \dots, \mathbf{u}_{N})$. We call a spline flow $\mathbf{F}: S([0, 1]^d) \times \Theta^{\times d} \to [0, 1]^d$ using as a feature map $(H_i)_{i \in [d]}$ a \emph{signature spline flow}. 
\end{definition}

As before, we define our objective to be to minimize an adapted version of the KL-divergence of our model density with respect to the true density. 
\begin{align*}
    \mathrm{aKL}(p, p_\theta) &\coloneqq \mathbb{E}_p\left[\mathbb{E}_p\left[\ln \dfrac{p(\mathbf{x}_{t+1} | \F_{t})}{p_{\theta}(\mathbf{x}_{t+1} | \F_{t})} \Big| \F_{t}\right]\right]
    \\
    &=-\mathbb{E}_p\left[\mathbb{E}_p\left[\ln p_{\theta}(\mathbf{x}_{t+1} | \F_{t}) | \F_{t}\right]\right] + K
    \\
    &= -\sum_{i=1}^d\sum_{k=1}^{N}
    \E_p\left[\E_p\left[ C_{i,k}(\mathbf{x}_{t+1})\ln \Delta_{k}(\mathbf{x}, \theta_i)\right]\big|\F_{t} \right] + K
\end{align*}
For a finite set of realizations $\lbrace (\mathbf{x}_{1}^{(j)}, \dots, \mathbf{x}_{t+1}^{(j)})\rbrace_{j=1}^M $ we obtain the the Monte Carlo approximation 
\begin{equation}
    \label{eq:sig_spline_loss}
    J(\theta) =\dfrac{1}{M} \sum_{i=1}^d \sum_{k=1}^{N} \sum_{j=1}^M C_{i,k}(\mathbf{x}_{t+1}^{(j)})
    \ln \Delta_{k}(\mathbf{x}^{(j)}, \theta_i)
\end{equation}
where we denote the parameters as $\theta = (\mathbf{U}_1, \dots, \mathbf{U}_d) \in \Theta^{\times d}$. 

The following theorem states that calibrating the cost functions parameters $\theta \in \Theta$ is convex. The proof can be found in \autoref{appendix:proofs}. 
\begin{restatable}[Convexity]{theorem}{convexity}
    \label{thm:convex}
    The objective function $J: \Theta \to \mathbb{R}$ is convex. 
\end{restatable}

When working with a limited size dataset optimizing with respect to the cost function $J$ may lead to an overfitting of the density's parameters. The following corollary shows that regularizing the model's parameters using a convex penalty function maintains the convexity property of the calibration problem. 
\begin{corollary}
    Let $\lambda > 0$ and let $\alpha: \Theta \to \R$ be a convex function of $\theta$. Then the regularized objective $J_\lambda(\theta) \coloneqq J(\theta) + \lambda \alpha(\theta)$ is convex. 
\end{corollary}

Note that the regularized objective includes the penalty functions $\alpha(\theta) = \| \theta \|_2^2$ and $\alpha(\theta) = \| \theta \|_1$, so we may retain the convexity of our objective through these standard penalty functions. The regularized objective can be particularly helpful to avoid overfitting in the context of time series generation where we may only have one single realization of the time series. 

We now present our main theoretical result, that the sig-spline construction is able to approximate the conditional transition density of any time series arbitrarily well, given a high enough signature truncation order, and sufficiently many knots in the spline. Again the proof can be found in \autoref{appendix:proofs}.
\begin{restatable}[Universality of Sig-Splines]{theorem}{sigsplines}
    \label{thm:universal_sig_splines}
    Let $\mathbf{x} \sim p$ be a $[0, 1]^d$-valued Markov process with $k$-lagged memory. Let $\varepsilon > 0$. Then there exists an order of truncation $L$, a number of bins $N$ and a set of linear functionals $\mathbf{U} \in \mathbb{R}^{N \times f(d, L)}$ such that for all paths $\mathbf{x} \in S([0, 1]^{d})$ of length $t$
    \begin{equation*}
        |p( \mathbf{x}_{t+1} | \F_t) -  p_\mathbf{W}( \mathbf{x}_{t+1} | \F_t) | < \varepsilon
    \end{equation*}
    with respect to the $L^2$ norm.
\end{restatable}

\section{Numerical results}
\label{sec:numerical}
This section delves into the evaluation of sig- and neural splines, examining their performance through a series of controlled data and real-world data experiments in subsections \ref{sec:var_dataset} to \ref{sec:spot_dataset}. Subsection \ref{sec:var_dataset} focuses on assessing generative performance using a VAR process with known parameters. In contrast, subsections \ref{sec:vol_dataset} and \ref{sec:spot_dataset} evaluate the performance of these splines on realized volatilities of multiple equity indices and spot prices, respectively.

\subsection{Experiment outline}
\subsubsection{Training}
The neural spline flow serves as the benchmark model to surpass in all experiments. Each neural spline flow consists of three hidden layers, each with 64 hidden dimensions which are used to construct the linear spline CDF. On the other hand, the sig-spline models are calibrated for orders of truncation ranging from 1 to 4, allowing us to observe how the performance of the sig-spline model varies with higher orders. Throughout all experiments, the models are calibrated using 2-dimensional time series data. The number of parameters used in each sig-spline model, assuming a 2-dimensional path, is reported in \autoref{table:sig_parameters}.

Before training, the dataset is divided into a train and test set. To account for the randomness in the train and test split and its potential impact on model performance, the models are calibrated using 10 different seeds. Performance metrics are computed as averages across all 10 calibration runs.

Within each calibration, \emph{early stopping} \cite{prechelt2002early} is applied to each individual conditional density estimator to mitigate overfitting on the train set. Early stopping halts the fitting of a conditional density estimator when the test set error increases consecutively for a certain number of times. In these experiments, the patience (the number of consecutive test set errors allowed before the fitting process stops) is set to 32. All models are trained using full-batch gradient descent.

\begin{table}[h!]\centering
    \begin{tabular}{lllll}
    \toprule  
    Order of truncation & 1 & 2 & 3 & 4 \\
    \midrule  
    Parameters & 512 & 1664 & 5120 & 15488  \\ 
    \bottomrule 
\end{tabular}
\caption{Number of parameters used in each sig-spline as a function of the order.}
\label{table:sig_parameters}
\vspace{-0.8cm}
\end{table}

\subsubsection{Evaluation}
After training, all calibrated models are evaluated using a set of test metrics. This involves sampling a batch of time series of length 4 from the calibrated model, computing standard statistics from the generated dataset, and comparing them with the empirical statistics of the real dataset. This comparison is done by calculating the difference between the statistics and applying the $l_1$ norm. Specifically, four statistics are compared for both the return process and the level process: the first two lags of the autocorrelation function (ACF), the skewness, the kurtosis, and the cross-correlation. Additionally, for the multi-asset spot return dataset, the ACF of the absolute returns is compared as a lower discrepancy would indicate that the generative model is capable of capturing volatility clustering. 

All numerical results can be found in \autoref{appendix:numerical}. Each table presents the performance metrics of the models, with the best-performing metric highlighted in bold font. \emph{It is important to note that the designation of the best-performing model is merely indicative, as in several cases, the presence of error bounds prevents the identification of a clear best-performing model.}

\subsection{Vector autoregression} 
\label{sec:var_dataset}
Assume a $d$-dimensional $\mathrm{VAR}(2)$ process $(\mathbf{y}_t)_t$ taking the form 
$$ 
\mathbf{y}_{t+1} = \mathbf{W}_1 \mathbf{y}_{t} + \mathbf{W}_2 \mathbf{y}_{t-1} + \mathbf{\Sigma}^{\frac{1}{2}} \mathbf{z}_{t+1} 
$$
where $\mathbf{W}_1, \mathbf{W}_2, \mathbf{\Sigma} \in \mathbb{R}^{d \times d}$ are matrices and $\mathbf{z}_t \sim \mathcal{N}(0, \mathbf{I})$ is an adapted normally distributed $l$-dimensional random variable. The process $(\mathbf{y}_t)_t$ is assumed to be \emph{latent} in the sense that it is not observed. The observed process $(\mathbf{x}_t)_t$ is assumed to be $D$-dimensional and defined as 
\begin{equation*}
    \mathbf{x}_t = F_\theta(\mathbf{y}_t), t \in \N 
\end{equation*}
where $F_\theta: \R^d \to \R^D$ is an unknown non-linear function. 

In the controlled experiment the dimensions $d=2$ and $D=8$ are assumed. The decoder $F_\theta: \R^2 \to \R^8$ is represented by a neural network with two hidden layers, $64$ hidden dimensions initialized randomly using the He initialization scheme \cite{he2015delving} and parametric ReLUs as activation functions and the VAR dynamic is governed by the autoregressive matrices 
$$ 
\mathbf{W}_1 = \begin{pmatrix}
    0.1 & 0.0 \\ 
    0.0 & 0.2
\end{pmatrix}
\quad 
\mathbf{W}_2 = \begin{pmatrix}
    0.6 & 0.0 \\ 
    0.0 & 0.3
\end{pmatrix}
$$
and the covariance matrix 
$$
\mathbf{\Sigma} = \begin{pmatrix}
    0.5 & 0.0 \\
    0.0 & 0.5
\end{pmatrix}.
$$

A total of 4096 lags of the latent process are sampled to create a simulated empirical dataset. These lags serve as the basis for generating the observed empirical time series using a randomly sampled decoder (refer to \autoref{fig:var_compressed_path}). Subsequently, an autoencoder is trained to compress the generated time series back to its original latent dimension. This compression step is employed to enhance scalability and demonstrate that autoencoders can produce a lower-dimensional representation of the time series.

\begin{figure}[h!]
    \begin{minipage}{.45\textwidth}
        \centering
        \includegraphics[width=\textwidth]{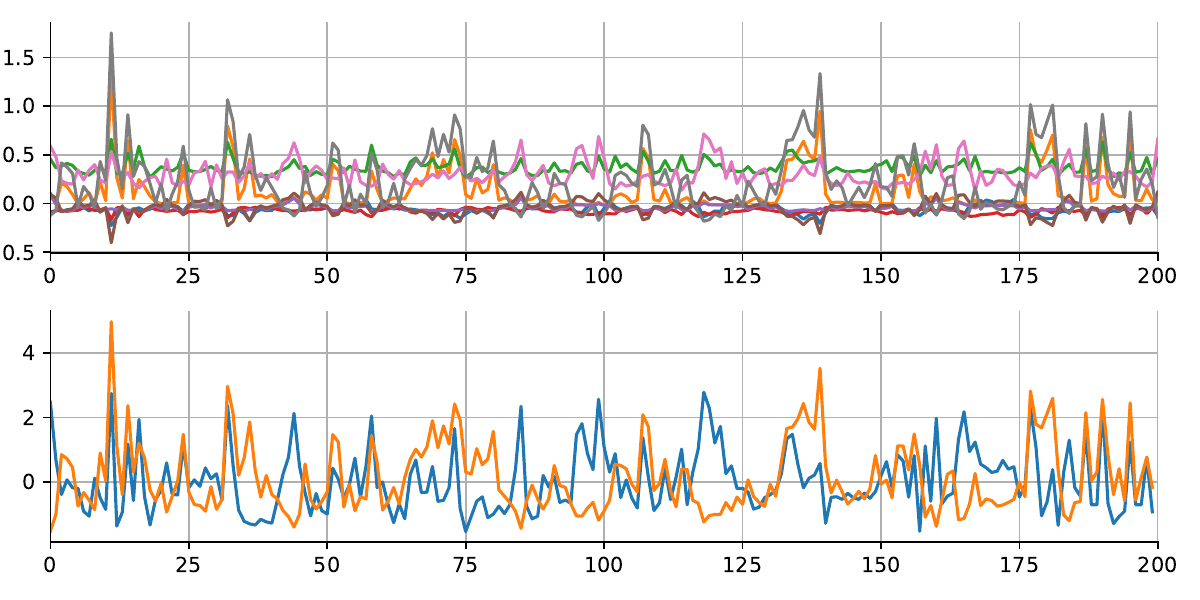}
        \vspace{-0.5cm}
        \caption{Sample of the original observed path (top) $(\mathbf x_t)_t$ \\and the compressed path (bottom) $(\mathbf y_t)_t$.}
        \label{fig:var_compressed_path}
    \end{minipage}
\end{figure}

Utilizing the calibrated autoencoder, the encoder component is applied to obtain the compressed time series. This compressed representation is then used to calibrate the conditional density approximators, enabling the modeling of the conditional density of the original time series based on the compressed data obtained from the autoencoder.

Tables \ref{table:var_compressed_level} - \ref{table:var_rtn} report the performance metrics for the compressed level and return process, and the observed level and return process respectively. Upon examining all the tables, it becomes evident that the performance of sig-spline models consistently improves with higher truncation orders. This holds particularly true for the autocorrelation function, revealing that a truncation order below 3 fails to adequately capture the proper dependence of the $\mathop{VAR}(2)$ process. 

When comparing all the tables, it becomes apparent that the neural spline model holds a slight advantage over the sig-spline model truncated at order 4. However, when considering the performance on the compressed process, the two models demonstrate very comparable performance.

\subsection{Realized volatilities}
\label{sec:vol_dataset}
The subsequent case study explores the effectiveness of the sig-spline model in approximating the dynamics of a real-world multivariate volatilities dataset derived from the MAN AHL Realized Library. To conduct the numerical evaluation, the med-realized volatilities of Standard \& Poor's 500 (SPX), Dow Jones Industrial Average (DJI), Nikkei 225 (N225), Euro STOXX 50 (STOXX50E), and Amsterdam Exchange (AEX) are extracted from the dataframe. These volatilities form a 5-dimensional time series spanning from January 1st, 2005, to December 31st, 2019. \autoref{fig:vols} provides a visual representation of the corresponding historical volatilities.

\sloppy
An observation from \autoref{fig:vols} reveals a strong correlation between both the levels and their returns (refer to \autoref{fig:vol_cc} for the corresponding cross-correlation matrices). Due to these high cross-correlations, an autoencoder is calibrated to compress the $5$-dimensional time series into a $2$-dimensional representation. This $2$-dimensional time series, depicted in \autoref{fig:vols}, serves as the basis for calibrating the neural and sig-spline conditional density estimators.

\begin{figure}[h!]
    \centering
    \includegraphics[width=\linewidth]{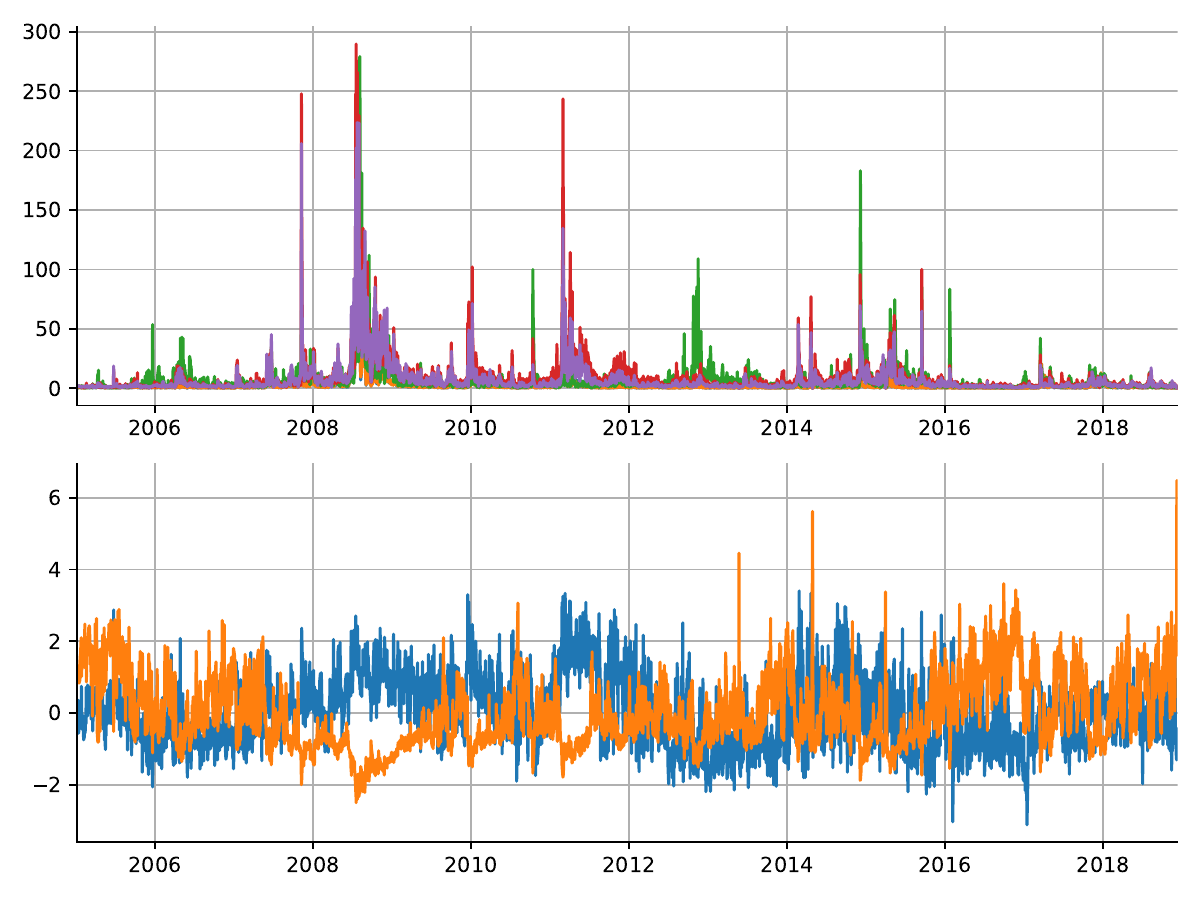}
    \vspace{-0.5cm}
    \caption{Historical realized med-volatilities (top) of the considered stock indices and compressed realized volatilities (bottom).}
    \label{fig:vols}
\end{figure}

\begin{figure}
    \centering
    \includegraphics[width=\linewidth]{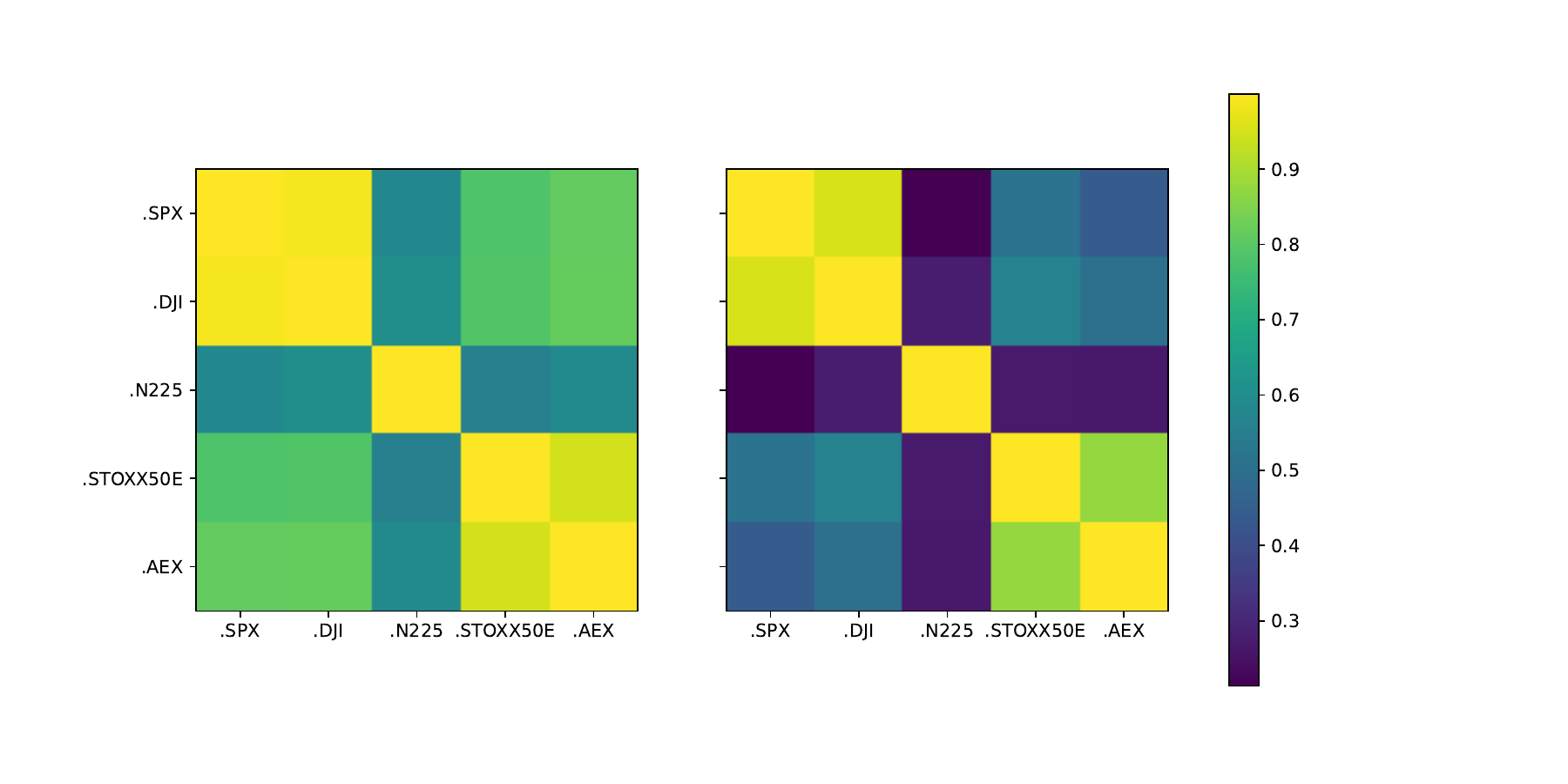}
    \vspace{-0.5cm}
    \caption{Cross-correlation matrices of the volatility levels (left) and returns (right).}
    \label{fig:vol_cc}
\end{figure}

The performance metrics for the compressed 2-dimensional return and level process are presented in \autoref{table:rv_compressed_level} and \autoref{table:rv_compressed_rtn}, while \autoref{table:rv_level} and \ref{table:rv_rtn} display the performance metrics for the original 5-dimensional process. Analysis of \autoref{table:rv_compressed_level} and \ref{table:rv_compressed_rtn} indicates that, in most metrics, the neural spline demonstrates superior performance compared to the sig spline for the compressed process. This suggests that the neural network exhibits better ability to capture complex dependence structures in more intricate real-world datasets. However, it is interesting to note that this performance advantage of the neural spline does not directly translate to an advantage in the observed process, as observed in \autoref{table:rv_level} and \ref{table:rv_rtn}. In fact, the signature spline truncated at order 4 frequently exhibited better performance compared to the neural spline flow.


It is noteworthy that \autoref{table:rv_level} and \autoref{table:rv_rtn} highlight the kurtosis metrics, which stand out prominently. This observation arises from the high kurtosis exhibited by the med-realized volatilities, indicating that both the neural and sig-spline models struggle to capture the heavy-tailed nature of the data.

\subsection{Multi-asset spot returns}
\label{sec:spot_dataset}
The following real-world multi-asset spot return dataset is sourced from the MAN AHL Realized Library. It focuses on the SPX and DJI stock indices, covering the period from January 1st, 2005, to December 31st, 2021 (refer to the top figure of \autoref{fig:multi_asset}).

During the calibration process, it was noted that sig-splines struggled to capture the strong cross-correlations in the returns of SPX and DJI. To address this issue, the spot time series underwent preprocessing by applying Principal Component Analysis (PCA) to the index returns, resulting in whitened returns. The preprocessed time series is depicted in \autoref{fig:multi_asset}. Subsequently, the models were calibrated using the transformed return series.

\autoref{table:multi_compressed_rtn} and \ref{table:multi_rtn} present the performance metrics for both the preprocessed return series and the observed return process. Notably, the cross-correlation metric for all sig-spline models closely compares with the neural spline baseline, thanks to the application of PCA transform during the preprocessing of the spot return series. Both tables demonstrate that sig- and neural splines perform relatively well, with sig-splines showcasing superior performance. Moreover, it is observed that higher orders of truncation lead to improved performance in autocorrelation metrics, which detect serial independence and volatility clustering.

\begin{figure}[h]
    \centering
    \includegraphics[width=\linewidth]{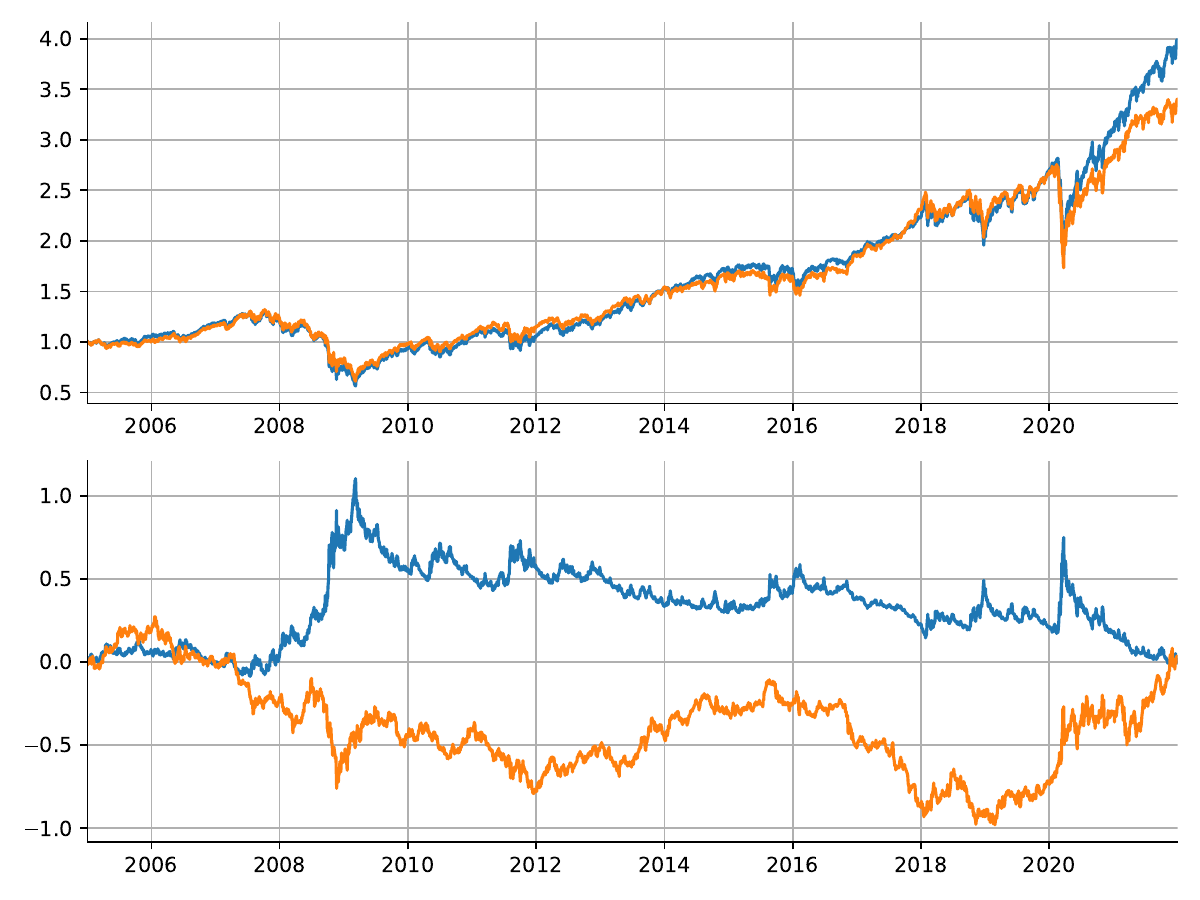}
    \caption{Normalized historical spot path of SPX (blue) and DJI (orange) (top plot) and the preprocessed path (bottom plot).}
    \label{fig:multi_asset}
\end{figure}

\section{Conclusion}
This paper introduced signature spline flows as a generative model for time series data. By employing the signature transform, sig-splines are constructed as an alternative to the neural networks used in neural spline flows. In \autoref{sec:signature_spline_flows}, we formally demonstrate the universal approximation capability of sig-splines, highlighting their ability to approximate any conditional density. Additionally, we establish the convexity of sig-spline calibration with respect to its parameters.

To assess their performance, we compared sig-splines with neural splines in \autoref{sec:numerical} using a simulated benchmark dataset and two real-world financial datasets. Our evaluation, based on standard test metrics, reveals that sig-splines perform comparably to neural spline flows.

The convexity and universality properties of sig-splines pique our interest in future research directions. We believe that pursuing these avenues could yield valuable insights and fruitful outcomes:
\begin{itemize}
    \item Sig-ICA (Signature Independent Component Analysis): Exploring the application of a signature-based ICA introduced in \cite{schell2023nonlinear} could enhance the scalability of the sig-spline generative model. A Sig-ICA preprocessing would allow calibrating a sig-spline model on each coordinate for each coordinate of the process, allowing for less parameters and more interpretablity. 
    \item Regularisation techniques: This paper has not extensively addressed methods for improving the generalization of sig-splines. Investigating regularization techniques specifically tailored to sig-splines could prove beneficial in enhancing their performance and robustness, particularly in scenarios with limited training data or complex dependencies.
\end{itemize}

\clearpage
\bibliography{references}

\clearpage
\appendix
\onecolumn 
\section{Proofs}
\label{appendix:proofs}

\subsection{Proof of Theorem \ref{thm:convex}}
\convexity*

\begin{proof}
    \footnote{The proof follows ideas from \cite*{MLR} paper \cite{MLR}. A short summary of the proof was uploaded by \citeauthor{TrungMLR} \cite{TrungMLR}.} 
    The sum of two convex functions remains convex. We therefore consider the cost function \eqref{eq:sig_spline_loss} without loss of generality for a single sample $\mathbf{x}=(\mathbf{x}_1, \dots, \mathbf{x}_{t+1})$ and only a single conditional density $i \in \lbrace 1, \dots, d \rbrace$; i.e. we consider the cost function of the $i^{th}$ conditional density $J^i: \bigtimes_{k=1}^NT^{(L)}((\R^{1+d})^*) \to \R$ defined as 
    \begin{align*}
        J^i(\mathbf{V}^i)
        &= -\sum_{k=1}^N \Big[{c}_k \Big(\langle \mathbf{v}_{k}, \mathbf{y}\rangle - \ln\big( \sum_{j=1}^N \exp(\langle \mathbf{v}_j, \mathbf{y} \rangle )\big)\Big)\Big]
    \end{align*}
    $\mathbf{y} = \Sig^{(L)}(\eta^i({\mathbf{x}}))$ is the signature of the masked path, and $\mathbf{c} \in \lbrace 0, 1 \rbrace^N$ are the indicator functions
    \begin{equation}
        \mathbf{c} = (\mathbf{1}_{[\alpha^i_{k-1}, \alpha^i_{k}]}(x_{t+1, i}))_{k \in \lbrace 1, \dots, N \rbrace }
    \end{equation}
    where $0=\alpha_0^i < \dots < \alpha_N^i=1$ is the predefined partition of the $x$-axis of the constructed CDF. We obtain the full loss function $J$ as 
    \begin{equation*}
        J(\theta)=J(\mathbf{V}^1, \dots, \mathbf{V}^d) = \sum_{i = 1}^d J^i(\mathbf{V}^i). 
    \end{equation*}
    For legibility we drop the dependence on $i$ and write $H(\mathbf{V})\coloneqq J^i(\mathbf{V}^i)$ instead. To ease the notation we project the linear functionals $ \mathbf{v}_i, i \in \lbrace 1, \dots, N \rbrace $ and the signature $\mathbf{y} \in T^{(L)}(\R^d)$ to the real-valued vector space $\R^K, K={f(L, 1+d)}$, thus the linear functionals can be thought of as \emph{weights} hereafter.

    To show that the cost function $H: \bigtimes_{i=1}^N \R^K \to \R$ is convex we derive the first- and second-order gradients with respect to the model parameters $\mathbf{V}=(\mathbf{v}_1, \dots, \mathbf{v}_N)$. The first-order derivative is defined for a single set of weights $\mathbf{v}_i, i \in \lbrace 1, \dots, N\rbrace$ is 
    \begin{equation*}
        \dfrac{\partial H(\mathbf{V})}{\partial \mathbf{v}_i} 
        = 
        \dfrac{\exp(\langle \mathbf{v}_i, \mathbf{y}\rangle )}{\sum_{j=1}^N\exp(\langle \mathbf{v}_j, \mathbf{y}\rangle)} \mathbf{y} -{c}_i \mathbf{y}
        = 
        {p}_i \mathbf{y} - {c}_i \mathbf{y}
    \end{equation*}
    where ${p}_i = \Phi(\Sig^{(L)} )$
    and can be expressed in matrix-vector notation as
    \begin{equation*}
        \nabla_{\mathbf{V}} H(\mathbf{V}) = (\mathbf{p} - \mathbf{c})\otimes \mathbf{y}
    \end{equation*}
    where $\otimes$ denotes the Kronecker product. 
    Furthermore, the second-order gradient is given for any $i, j \in \lbrace 1, \dots, N \rbrace $ as 
    \begin{equation*}
        \dfrac{\partial H(\mathbf{V})}{\partial \mathbf{v}_j \partial \mathbf{v}_i^T} 
        = 
        \dfrac{\partial}{\partial \mathbf{v}_j}\left({p}_i \mathbf{y} - {c}_i \mathbf{y}\right)^T 
        = 
        \dfrac{\partial}{\partial \mathbf{v}_j} p_i \mathbf{y}^T 
        = 
        (\delta_{ij} {p}_i - {p}_i {p}_j) \mathbf{y} \mathbf{y}^T
    \end{equation*}
    which in matrix-vector notation can be expressed as
    \begin{equation}
        \label{eq:hessian}
        \nabla^2_{\mathbf{V}} H(\mathbf{V}) = (\mathbf{D}(\mathbf{p}) - \mathbf{p}\mathbf{p}^T)\otimes \mathbf{y} \mathbf{y}^T 
    \end{equation}
    where $\mathbb{D}(\mathbb{p})$ denotes the weight matrix with diagonal entries 
    
    Finally, we need to show that the Hessian of the cost function $H$ with respect to $\mathbf{V}$ is positive semidefinite. To demonstrate this, first recall that if the Eigenvalues of the square matrices $ \mathbf{A} \in \R^{n \times n}, \mathbf{B} \in \R^{m\times m}$ are $\alpha_i, i \in \lbrace 1, \dots, n \rbrace$ and $\beta_j, j\in \lbrace 1, \dots, m \rbrace$ respectively, then the Eigenvalues of $\mathbf{A} \otimes \mathbf{B}$ are $\alpha_i\beta_j, i \in \lbrace 1, \dots, n \rbrace, j \in \lbrace 1, \dots, m \rbrace$. Thus, it suffices to show that the matrices $(\mathbf{D}(\mathbf{p}) - \mathbf{p}\mathbf{p}^T) \in \R^{N \times N}$ and $\mathbf{y} \mathbf{y}^T \in \R^{K \times K} $ are positive semidefinite. For the latter matrix this is straightforward since for any $\mathbf{a} \in \R^K$ we have $ \mathbf{a}^T\mathbf{y}\mathbf{y}^T\mathbf{a} = (\mathbf{a}^T\mathbf{y})^2 \geq 0 $. For the former we can show that it is diagonally-dominant and conclude that it is positive semidefinite \textbf{CITE}.
    Thus, 
    \begin{equation}
        \lambda_{\min} (\nabla^2_{\mathbf{V}} H(\mathbf{V})) = \lambda_{\min}(\mathbf{D}(\mathbf{p}) - \mathbf{p}\mathbf{p}^T)\cdot \lambda_{\min}(\mathbf{y} \mathbf{y}^T) \geq 0 
    \end{equation}
    which concludes the proof. 
\end{proof}

\begin{remark}
    Note that the convexity of the cost function $J: \Theta \to \R $ does not hold if one makes the $\alpha$-partition. 
\end{remark}

\subsection{Proof of Theorem \ref{thm:universal_sig_splines}}

\sigsplines*
    
\begin{proof}
    We first note that by factoring the true and model densities into the product of the conditionals $p({x}_i | \mathbf{x}_{:i-1}, \F_t)$, we can write (dropping the explicit reference to the filtration to ease notation) 

    \begin{equation*}
        \left| p( \mathbf{x}_{t+1} | \F_t) -  p_\mathbf{W}( \mathbf{x}_{t+1} | \F_t) \right| = \left| \prod_{i=1}^d p({x}_i | \mathbf{x}_{:i-1}) - \prod_{i=1}^d p_\mathbf{W}({x}_i | \mathbf{x}_{:i-1}) \right|
    \end{equation*}

    To use this factorisation, we make use of the following lemma.

    \begin{lemma}
        Let $a_1, \ldots, a_d$ and $b_1, \ldots, b_d$ be two bounded sequences of real numbers. Then there exists $M \in \R$ such that 

        \begin{equation*}
            \left| \prod_{i=1}^d a_i - \prod_{i=1}^d b_i \right| \leq M \sum_{i=1}^d \left| a_i - b_i \right|
        \end{equation*}
    \end{lemma}

    \begin{proof}
        Write $\prod_{i=1}^d a_i - \prod_{i=1}^d b_i = \sum_{i=1}^d a_1 \cdots a_{i-1}(a_i - b_i)b_{i+1} \cdots b_d$. Then taking absolute values and applying the triangle inequality, taking $M = \max_{i=1,\ldots, d} |a_1 \cdots a_{i-1}b_{i+1} \cdots b_d|$ we obtain the result.
    \end{proof}
    
    Hence, there exists $M \in \R$ such that we can write 

    \begin{equation*}
        \left| \prod_{i=1}^d p({x}_i | \mathbf{x}_{:i-1}) - \prod_{i=1}^d p_\mathbf{W}({x}_i | \mathbf{x}_{:i-1}) \right| \leq M \sum_{i=1}^d \left| p({x}_i | \mathbf{x}_{:i-1}) - p_\mathbf{W}({x}_i | \mathbf{x}_{:i-1}) \right| 
    \end{equation*}

    Thus, since to prove the claim, it would be sufficient to establish that we can approximate each conditional arbitrarily well, without loss of generality we can focus on the univariate case $d=1$. For any $N \in \N$, define the piecewise constant approximation of the true conditional density as 
    \begin{equation*}
        \hat{p}(x_{t+1}| \F_t) \coloneqq \prod_{j=1}^N \hat{p}_j^{c_j} 
    \end{equation*}
    with 
    \begin{equation*}
        \hat{p}_j \coloneqq N \int_{\frac{j-1}{N}}^{\frac{j}{N}} p(x_{t+1}| \F_t) \ \mathrm{d} x_{t+1}  
    \end{equation*}
    and $c_j$ defined as above. Such piecewise constant functions are known to be dense in $L^2[0, 1]$, hence given any $\varepsilon > 0$ we can find $N$ such that 

    \begin{equation*}
        | p(x_{t+1} | \F_t) -  \hat{p}(x_{t+1} | \F_t) | < \varepsilon / 2
    \end{equation*}
    according to the $L^2$ norm. For such $N$, consider the log probabilities $q_j = \ln \hat{p}_j$. Due to the universality of signatures, for each $j=1, \ldots, N$ we can find a trunctation order $L_j \coloneqq L_j(N)$ and a weight vector $\mathbf{w}_j \in \R^{f(d, L_j)}$ such that $| q_j - \mathbf{w}_j^T \mathbf{y}| < \varepsilon / 2N$. Define $L = \max \{ L_j : j = 1, \ldots, N \}$ and $K=f(d,L)$, then take the weight matrix $\mathbf{W} \in \bigtimes_{i=1}^N T^{(L)}(\R^{1+d})$ to have rows $\mathbf{w}_j$ concatenated with zeros whenever $L_j < L$. Thus, writing $\mathbf{q} = (q_1, \ldots, q_N)$ we have 
    
    \begin{equation*}
        | \mathbf{q} - \mathbf{W} \mathbf{y} |  \leq \sum_{j=1}^N | q_j - \mathbf{w}_j^T \mathbf{y} | \leq \varepsilon / 2
    \end{equation*}

    Defining the vectors $\hat{\mathbf{p}} = \softmax(\mathbf{q})$ and $\mathbf{p}_\mathbf{W} =\softmax(\mathbf{Wy})$ similarly, by the Lipschitz property of the softmax function, with constant less than $1$ for $N > 2$, we have

    \begin{equation*}
        | \hat{\mathbf{p}} - \mathbf{p}_\mathbf{W} | \leq | \mathbf{q} - \mathbf{W} \mathbf{y} | \leq \varepsilon / 2
    \end{equation*}

    Hence, with an application of the triangle inequality we obtain our result
    \begin{equation*}
        | p(x_{t+1}| \F_t) - p_\mathbf{W}(x_{t+1}| \F_t) | \leq \varepsilon \ .
    \end{equation*}

\end{proof}

\clearpage 
\counterwithin{table}{subsection}
\renewcommand\thetable{\thesubsection.\arabic{table}}
\clearpage
\onecolumn
\section{Numerical results}
\label{appendix:numerical}
\subsection{VAR model}
\begin{table}[h!]\centering\begin{tabular}{lllll}
    \toprule Model / Test metrics & $|\rho^x_h - \rho^x_g|_1$ & $|\kappa^x_h - \kappa^x_g|_1$ & $|s_h^x - s_g^x|_1$ & $|\Sigma_h^x - \Sigma_g^x|_1$ \\ \midrule Sig-Spline (Order 1)& $ 0.0147 \pm 0.0063$ & $ \textbf{0.3478} \pm 0.1675$ & $ 0.3072 \pm 0.0425$ & $ {0.0108} \pm 0.0065$ 
    \\ \midrule Sig-Spline (Order 2)& $ {0.0133} \pm 0.0059$ & $ 2.8459 \pm 1.4922$ & $ 0.3402 \pm 0.1123$ & $ {0.0145} \pm 0.0098$ \\ \midrule Sig-Spline (Order 3)& $ {0.0108} \pm 0.0037$ & $ 0.6097 \pm 0.2611$ & $ {0.1093} \pm 0.0255$ & $ 0.0162 \pm 0.0095$ 
    \\ \midrule Sig-Spline (Order 4)& $ {0.0122} \pm 0.0064$ & $ 0.5694 \pm 0.1948$ & $ {0.1075} \pm 0.0406$ & $ 0.0168 \pm 0.0119$ 
    \\ \midrule Neural spline flow& $\textbf{0.0102} \pm 0.0038$ & ${0.4404} \pm 0.3541$ & $\textbf{0.0979} \pm 0.0357$ & $\textbf{0.0107} \pm 0.0053$ \\ \bottomrule \end{tabular}\caption{Performance metrics of the compressed level process.}\label{table:var_compressed_level}
\end{table}
\begin{table}[h!]\centering\begin{tabular}{lllll}\toprule Model / Test metrics & $|\rho^r_h - \rho^r_g|_1$ & $|\kappa^r_h - \kappa^r_g|_1$ & $|s_h^r - s_g^r|_1$ & $|\Sigma_h^r - \Sigma_g^r|_1$ 
    \\ \midrule Sig-Spline (Order 1)& $ 0.0579 \pm 0.0034$ & $ 0.3888 \pm 0.1913$ & $ 0.0867 \pm 0.0557$ & $ 0.0216 \pm 0.0124$ 
    \\ \midrule Sig-Spline (Order 2)& $ 0.0122 \pm 0.0050$ & $ 0.3412 \pm 0.1222$ & $ 0.0708 \pm 0.0386$ & $ 0.0154 \pm 0.0096$ 
    \\ \midrule Sig-Spline (Order 3)& $ {0.0084} \pm 0.0028$ & $ 0.2044 \pm 0.0481$ & $ 0.0603 \pm 0.0277$ & $ 0.0139 \pm 0.0171$ \\ \midrule Sig-Spline (Order 4)& $ {0.0098} \pm 0.0041$ & $ 0.2107 \pm 0.0754$ & $ \textbf{0.0569} \pm 0.0310$ & $ \textbf{0.0113} \pm 0.0109$ 
    \\ \midrule Neural spline flow& $\textbf{0.0076} \pm 0.0036$ & $\textbf{0.1972} \pm 0.0691$ & $0.0585 \pm 0.0275$ & $0.0134 \pm 0.0105$ 
    \\ \bottomrule \end{tabular}\caption{Performance metrics of the compressed return process.}\label{table:var_compressed_rtn}\end{table}
\begin{table}[h!]\centering\begin{tabular}{lllll}\toprule Model / Test metrics & $|\rho^x_h - \rho^x_g|_1$ & $|\kappa^x_h - \kappa^x_g|_1$ & $|s_h^x - s_g^x|_1$ & $|\Sigma_h^x - \Sigma_g^x|_1$ \\ \midrule Sig-Spline (Order 1)& $ 0.0398 \pm 0.0052$ & $ 4.2277 \pm 1.0256$ & $ 0.9135 \pm 0.1187$ & $ 0.1976 \pm 0.0122$ \\ \midrule Sig-Spline (Order 2)& $ 0.0479 \pm 0.0105$ & $ 78.8745 \pm 58.4826$ & $ 3.5408 \pm 1.8149$ & $ 0.1939 \pm 0.0278$ \\ \midrule Sig-Spline (Order 3)& $ 0.0297 \pm 0.0051$ & $ 5.5644 \pm 1.3769$ & $ 0.7894 \pm 0.1273$ & $ 0.1197 \pm 0.0105$ \\ \midrule Sig-Spline (Order 4)& $ 0.0322 \pm 0.0067$ & $ 7.6248 \pm 1.8894$ & $ 0.8857 \pm 0.1504$ & $ 0.1118 \pm 0.0083$ \\ \midrule Neural spline flow& $\textbf{0.0157} \pm 0.0041$ & $\textbf{2.9326} \pm 1.2577$ & $\textbf{0.4276} \pm 0.1100$ & $\textbf{0.0730} \pm 0.0084$ \\ \bottomrule \end{tabular}\caption{Performance metrics of the observed level process.}\label{table:var_level}\end{table}
\begin{table}[h!]\centering\begin{tabular}{lllll}\toprule Model / Test metrics & $|\rho^r_h - \rho^r_g|_1$ & $|\kappa^r_h - \kappa^r_g|_1$ & $|s_h^r - s_g^r|_1$ & $|\Sigma_h^r - \Sigma_g^r|_1$ \\ \midrule Sig-Spline (Order 1)& $ 0.0577 \pm 0.0058$ & $ 3.6188 \pm 0.6704$ & $ 0.4141 \pm 0.1134$ & $ 0.1707 \pm 0.0060$ \\ \midrule Sig-Spline (Order 2)& $ 0.0461 \pm 0.0173$ & $ 8.9617 \pm 3.3185$ & $ 0.9355 \pm 0.2366$ & $ 0.1232 \pm 0.0063$ \\ \midrule Sig-Spline (Order 3)& $ 0.0196 \pm 0.0036$ & $ 4.5679 \pm 1.2919$ & $ 0.5792 \pm 0.1436$ & $ 0.0928 \pm 0.0075$ \\ \midrule Sig-Spline (Order 4)& $ 0.0231 \pm 0.0047$ & $ 5.4555 \pm 1.1996$ & $ 0.5990 \pm 0.0888$ & $ 0.0845 \pm 0.0067$ \\ \midrule Neural spline flow& $\textbf{0.0127} \pm 0.0052$ & $\textbf{2.6636} \pm 0.7675$ & $\textbf{0.2558} \pm 0.0933$ & $\textbf{0.0702} \pm 0.0060$ \\ \bottomrule \end{tabular}\caption{Performance metrics of the observed return process.}\label{table:var_rtn}\end{table}

\clearpage

\subsection{Realized volatilities}
\begin{table}[h]\centering\begin{tabular}{lllll}\toprule Model / Test metrics & $|\rho^x_h - \rho^x_g|_1$ & $|\kappa^x_h - \kappa^x_g|_1$ & $|s_h^x - s_g^x|_1$ & $|\Sigma_h^x - \Sigma_g^x|_1$ \\ \midrule Sig-Spline (Order 1)& $ {0.0074} \pm 0.0028$ & $ {0.6196} \pm 0.2898$ & $ 0.1844 \pm 0.0479$ & $ {0.0121} \pm 0.0071$ \\ \midrule Sig-Spline (Order 2)& $ {0.0093} \pm 0.0018$ & $ 1.4172 \pm 0.3019$ & $ 0.2043 \pm 0.0686$ & $ 0.0232 \pm 0.0093$ \\ \midrule Sig-Spline (Order 3)& $ 0.0141 \pm 0.0038$ & $ 1.3337 \pm 0.1564$ & $ 0.3197 \pm 0.0741$ & $ 0.0546 \pm 0.0087$ \\ \midrule Sig-Spline (Order 4)& $ 0.0143 \pm 0.0029$ & $ 1.2699 \pm 0.2692$ & $ 0.3211 \pm 0.0635$ & $ 0.0507 \pm 0.0116$ \\ \midrule Neural spline flow& $\textbf{0.0064} \pm 0.0031$ & $\textbf{0.5631} \pm 0.0912$ & $\textbf{0.0641} \pm 0.0336$ & $\textbf{0.0088} \pm 0.0078$ \\ \bottomrule \end{tabular}\caption{Performance metrics of the compressed level process.}\label{table:rv_compressed_level}\end{table}
\begin{table}[h]\centering\begin{tabular}{lllll}\toprule Model / Test metrics & $|\rho^r_h - \rho^r_g|_1$ & $|\kappa^r_h - \kappa^r_g|_1$ & $|s_h^r - s_g^r|_1$ & $|\Sigma_h^r - \Sigma_g^r|_1$ \\ \midrule Sig-Spline (Order 1)& $ 0.1055 \pm 0.0038$ & $ {1.0656} \pm 0.4424$ & $ {0.1556} \pm 0.1015$ & $ 0.0424 \pm 0.0080$ 
    \\ \midrule Sig-Spline (Order 2)& $ 0.0375 \pm 0.0072$ & $ 1.9738 \pm 0.3338$ & $ {0.1544} \pm 0.0628$ & $ 0.0170 \pm 0.0078$ \\ \midrule Sig-Spline (Order 3)& $ 0.0424 \pm 0.0048$ & $ {1.6472} \pm 0.2757$ & $ \textbf{0.1242} \pm 0.0324$ & $ {0.0093} \pm 0.0081$ 
    \\ \midrule Sig-Spline (Order 4)& $ 0.0357 \pm 0.0072$ & $ {1.4642} \pm 0.3115$ & $ 0.2270 \pm 0.0471$ & $ \textbf{0.0076} \pm 0.0073$ \\ \midrule Neural spline flow& $\textbf{0.0177} \pm 0.0044$ & $\textbf{1.0643} \pm 0.6016$ & ${0.1548} \pm 0.0862$ & ${0.0108} \pm 0.0075$ \\ \bottomrule \end{tabular}\caption{Performance metrics of the compressed return process.}\label{table:rv_compressed_rtn}\end{table}
\begin{table}[h]\centering\begin{tabular}{lllll}\toprule Model / Test metrics & $|\rho^x_h - \rho^x_g|_1$ & $|\kappa^x_h - \kappa^x_g|_1$ & $|s_h^x - s_g^x|_1$ & $|\Sigma_h^x - \Sigma_g^x|_1$ 
    \\ \midrule Sig-Spline (Order 1)& $ {0.1172} \pm 0.0552$ & $ 8343.9745 \pm 3288.4139$ & $ 74.3962 \pm 18.1762$ & $ 0.1318 \pm 0.0328$ 
    \\ \midrule Sig-Spline (Order 2)& $ \textbf{0.0972} \pm 0.0481$ & $ 4585.1863 \pm 3288.8783$ & $ 48.2547 \pm 19.9321$ & $ 0.1261 \pm 0.0409$ 
    \\ \midrule Sig-Spline (Order 3)& $ 0.1904 \pm 0.0573$ & $ {2949.4543} \pm 1667.5847$ & $ {35.1558} \pm 11.1678$ & $ \textbf{0.1099} \pm 0.0146$ 
    \\ \midrule Sig-Spline (Order 4)& $ 0.1775 \pm 0.0536$ & $ \textbf{2479.8028} \pm 1398.8296$ & $ \textbf{31.7946} \pm 10.3507$ & $ {0.1159} \pm 0.0198$ 
    \\ \midrule Neural spline flow& ${0.1147} \pm 0.0895$ & $5403.4121 \pm 4172.1582$ & $52.9103 \pm 24.0138$ & ${0.1177} \pm 0.0279$ \\ \bottomrule \end{tabular}\caption{Performance metrics of the observed level process.}\label{table:rv_level}\end{table}
\begin{table}[h!]\centering\begin{tabular}{lllll}\toprule Model / Test metrics & $|\rho^r_h - \rho^r_g|_1$ & $|\kappa^r_h - \kappa^r_g|_1$ & $|s_h^r - s_g^r|_1$ & $|\Sigma_h^r - \Sigma_g^r|_1$ \\ \midrule Sig-Spline (Order 1)& $ 0.1258 \pm 0.0681$ & $ 8598.7035 \pm 3035.6188$ & $ 69.0059 \pm 29.2961$ & $ 0.2549 \pm 0.0722$ \\ \midrule Sig-Spline (Order 2)& $ 0.1353 \pm 0.0698$ & $ 4300.6591 \pm 2254.0125$ & $ 48.5637 \pm 14.5840$ & $ \textbf{0.2156} \pm 0.0416$ \\ \midrule Sig-Spline (Order 3)& $ {0.0667} \pm 0.0322$ & $ {1912.7057} \pm 1028.8896$ & $ {21.6471} \pm 11.3284$ & $ 0.2362 \pm 0.0280$ \\ \midrule Sig-Spline (Order 4)& $ \textbf{0.0622} \pm 0.0349$ & $ \textbf{1673.3185} \pm 762.6186$ & $ \textbf{20.4442} \pm 7.3065$ & $ 0.2296 \pm 0.0409$ \\ \midrule Neural spline flow& $0.1084 \pm 0.0518$ & $6216.6590 \pm 3531.5966$ & $61.3187 \pm 25.7093$ & $0.2419 \pm 0.0678$ \\ \bottomrule \end{tabular}\caption{Performance metrics of the observed return process.}\label{table:rv_rtn}\end{table}

\clearpage
\subsection{Multi-asset spot returns}
\begin{table}[h!]\centering\begin{tabular}{llllll}\toprule Model / Test metrics & $|\rho^r_h - \rho^r_g|_1$ & $|\kappa^r_h - \kappa^r_g|_1$ & $|s_h^r - s_g^r|_1$ & $|\Sigma_h^r - \Sigma_g^r|_1$ & $|\rho^{|r|}_h - \rho^{|r|}_g|_1$ \\ \midrule Sig-Spline (Order 1)& $ 0.0140 \pm 0.0026$ & $ 0.7377 \pm 0.4659$ & $ 0.1238 \pm 0.0557$ & $ 0.0043 \pm 0.0029$ & $ 0.0140 \pm 0.0026$ \\ \midrule Sig-Spline (Order 2)& $ 0.0115 \pm 0.0039$ & $ 0.7787 \pm 0.5196$ & $ 0.1056 \pm 0.0391$ & $ 0.0036 \pm 0.0023$ & $ 0.0115 \pm 0.0039$ \\ \midrule Sig-Spline (Order 3)& $ 0.0103 \pm 0.0064$ & $ 0.9353 \pm 0.4326$ & $ \textbf{0.0995} \pm 0.0457$ & $ 0.0050 \pm 0.0039$ & $ 0.0103 \pm 0.0064$ \\ \midrule Sig-Spline (Order 4)& $ \textbf{0.0085} \pm 0.0058$ & $ 0.8250 \pm 0.3648$ & $ 0.0998 \pm 0.0577$ & $ 0.0041 \pm 0.0035$ & $ \textbf{0.0085} \pm 0.0058$ \\ \midrule Neural spline flow& $0.0145 \pm 0.0017$ & $\textbf{0.7160} \pm 0.3343$ & $0.1176 \pm 0.0621$ & $\textbf{0.0017} \pm 0.0010$ & $0.0145 \pm 0.0017$ \\ \bottomrule \end{tabular}\caption{Performance metrics of the preprocessed return process.}\label{table:multi_compressed_rtn}\end{table}

\begin{table}[h!]\centering\begin{tabular}{llllll}\toprule Model / Test metrics & $|\rho^r_h - \rho^r_g|_1$ & $|\kappa^r_h - \kappa^r_g|_1$ & $|s_h^r - s_g^r|_1$ & $|\Sigma_h^r - \Sigma_g^r|_1$ & $|\rho^{|r|}_h - \rho^{|r|}_g|_1$ \\ \midrule Sig-Spline (Order 1)& $ 0.0307 \pm 0.0037$ & $ 14.8671 \pm 0.9075$ & $ 0.8110 \pm 0.1067$ & $ 0.0057 \pm 0.0004$ & $ 0.0307 \pm 0.0037$ \\ \midrule Sig-Spline (Order 2)& $ 0.0259 \pm 0.0078$ & $ 15.3513 \pm 0.7230$ & $ 0.8226 \pm 0.1230$ & $ 0.0056 \pm 0.0004$ & $ 0.0259 \pm 0.0078$ \\ \midrule Sig-Spline (Order 3)& $ 0.0234 \pm 0.0074$ & $ 15.2616 \pm 0.8126$ & $ 0.7708 \pm 0.1158$ & $ \textbf{0.0055} \pm 0.0004$ & $ 0.0234 \pm 0.0074$ \\ \midrule Sig-Spline (Order 4)& $ \textbf{0.0211} \pm 0.0087$ & $ 15.4545 \pm 0.6326$ & $ 0.7699 \pm 0.0941$ & $ 0.0056 \pm 0.0003$ & $ \textbf{0.0211} \pm 0.0087$ \\ \midrule Neural spline flow& $0.0345 \pm 0.0021$ & $\textbf{13.6939} \pm 1.3599$ & $\textbf{0.6839} \pm 0.1220$ & $0.0057 \pm 0.0004$ & $0.0345 \pm 0.0021$ \\ \bottomrule \end{tabular}\caption{Performance metrics of the observed return process.}\label{table:multi_rtn}\end{table}

\end{document}